\documentclass{article}
\usepackage{hyperref}
\usepackage{url}
\usepackage{graphicx}
\usepackage{booktabs}
\usepackage{microtype}
\usepackage[normalem]{ulem}
\usepackage{multirow}
\usepackage{amsmath}
\usepackage{amssymb}
\usepackage{tabularx}
\usepackage{array}
\usepackage{framed}
\usepackage{color}
\usepackage{subcaption}
\usepackage{wrapfig}
\usepackage{algorithm}
\usepackage[algo2e,vlined]{algorithm2e}
\usepackage{amsthm}
\usepackage{natbib}

\usepackage{microtype}
\usepackage{color}
\providecommand{\se}[1]{}
\providecommand{\vk}[1]{}
\providecommand{\en}[1]{}
\providecommand{\substi}[1]{{\textcolor{blue}{#1}}}

\newtheorem{proposition}{Proposition}
\newtheorem{lemma}{Lemma}
\newtheorem{assumption}{Assumption}
\usepackage{mathtools}

\begin{document}

%
%
%

\title{Quantifying and Understanding Adversarial Examples in Discrete Input Spaces}

\author{%
  Volodymyr Kuleshov\textsuperscript{1}, Evgenii Nikishin\textsuperscript{2}\\
  Shantanu Thakoor\textsuperscript{3}, Tingfung Lau\textsuperscript{4}, Stefano Ermon\textsuperscript{5} \\
  \texttt{kuleshov@cornell.edu, nikishin.evg@gmail.com,}\\
  \texttt{shanu.thakoor@gmail.com,ldf921@126.com,ermon@cs.stanford.edu} \\  
  \textsuperscript{1}Cornell Tech, New York, NY\\
  \textsuperscript{2}University of Montreal, Montreal, QC\\
  \textsuperscript{3}Deepmind, London, UK\\
  \textsuperscript{4}Carnegie-Mellon University, Pittsburg, PA\\
  \textsuperscript{5}Stanford University, Stanford, CA\\
}

\date{}
\maketitle

\begin{abstract}
Modern classification algorithms are susceptible to adversarial examples---perturbations to inputs that cause the algorithm to produce undesirable behavior.  In this work, we seek to understand and extend adversarial examples across domains in which inputs are discrete, particularly across new domains, such as computational biology. As a step towards this goal, we formalize a notion of synonymous adversarial examples that applies in any discrete setting and describe a simple domain-agnostic algorithm to construct such examples. We apply this algorithm across multiple domains---including sentiment analysis and DNA sequence classification---and find that it consistently uncovers adversarial examples. 
We seek to understand their prevalence theoretically and we attribute their existence to spurious token correlations, a statistical phenomenon that is specific to discrete spaces. Our work is a step towards a domain-agnostic treatment of discrete adversarial examples analogous to that of continuous inputs.
\end{abstract}

\section{Introduction}
Modern machine learning algorithms are susceptible to adversarial examples---maliciously crafted inputs that fool the algorithm into producing undesirable behavior.
Adversarial examples arise in image classification \citep{szegedy2014intriguing}, speech recognition \citep{carlini2016hidden}, reinforcement learning \citep{behzadan2017vulnerability} and in other domains.
%
In many applications of machine learning---including natural language processing \citep{pang2008opinion,bahdanau2014neural}, genomics \citep{genomicsSurvey}, network science \citep{graphsSurvey}, and others---the inputs to a model are discrete, 
motivating several independent lines of work in these domains \citep{alzantot2018generating,morris2020textattack,zugner2018adversarial,wang2019attacking}.
%
Most of this work is domain-specific, involving definitions and analyses hand-crafted for each domain. 

The goal of this paper is to extend adversarial examples to new discrete domains, particularly to computational biology. To achieve this goal, we create a domain-agnostic framework that  features a general definition and an algorithm for finding adversarial examples. 
We demonstrate the utility and the generality of this approach by using it to identify examples in new domains (including gene and protein classification), as well as by obtaining strong performance in a well-studied domain, text classification.
We also study this approach theoretically, and find a new domain-agnostic cause for the existence of discrete adversarial examples.


Specifically, we define a notion of synonymous adversarial examples that applies in any discrete setting and that is constructed by perturbing inputs within a natural and domain-specific equivalence class of inputs such that they are misclassified.
Equivalence classes are defined via a user-specified similarity function $d(x_1,x_2)$, which is small when inputs $x_1,x_2$ are synonymous. Such similarity functions and their equivalence classes arise naturally in many domains. 
For example, if $x_1, x_2$ are graphs, we may say that $d(x_1, x_2) = 0$ if $x_1, x_2$ are isomorphic to each other. If $x_1$ and $x_2$ are words, we may gauge the extent to which they are synonyms.

Next, we describe a simple greedy algorithm that can construct such examples,
and apply it in a wide range of domains, including 
new problems in computational biology,
such as DNA sequence classification. 
Typical synonymous adversarial examples are created by replacing 10-30\% of symbols in a discrete sequence, and affect both deep neural networks as well as shallow linear models. They are also highly prevalent: up to 90\% of discrete inputs admit adversarial perturbations; furthermore, these perturbations retain a degree of transferability across models.

We seek to understand the prevalence of adversarial examples theoretically, and we attribute
their existence to spurious token correlations, a statistical phenomenon that is specific to discrete spaces and that is caused by an imbalance between the dimensionality of a problem (e.g., the size of a natural language vocabulary) and the size of the dataset. 
In brief, if the vocabulary of possible discrete tokens is very large, many tokens will accidentally correlate with a label given a finite dataset. These can be used to replace valid tokens to induce the correct class prediction to flip. We provide a theoretical analysis of this phenomenon, including bounds on the prevalence of spuriously correlated tokens.

\paragraph{Contributions.}
(1) We propose a framework of synonymous adversarial examples (AEs) inspired by previous work in NLP; it  facilitates (a)  studying discrete AEs in a domain-agnostic way and (b) identifying AEs in new domains while minimizing domain-specific work.
(2) We use this framework to define and identify AEs in a new domain, computational biology, on tasks such as gene and protein classification.
(3) We also use the framework to study discrete AEs theoretically and identify a novel cause for discrete AEs: spurious token correlations.
(4) Doing so, we lay the groundwork for identifying and understanding discrete AEs in new domains more quickly and more easily, both in computational biology and beyond.


\begin{figure*}[htb!]
\begin{center}
\includegraphics[width=16cm]{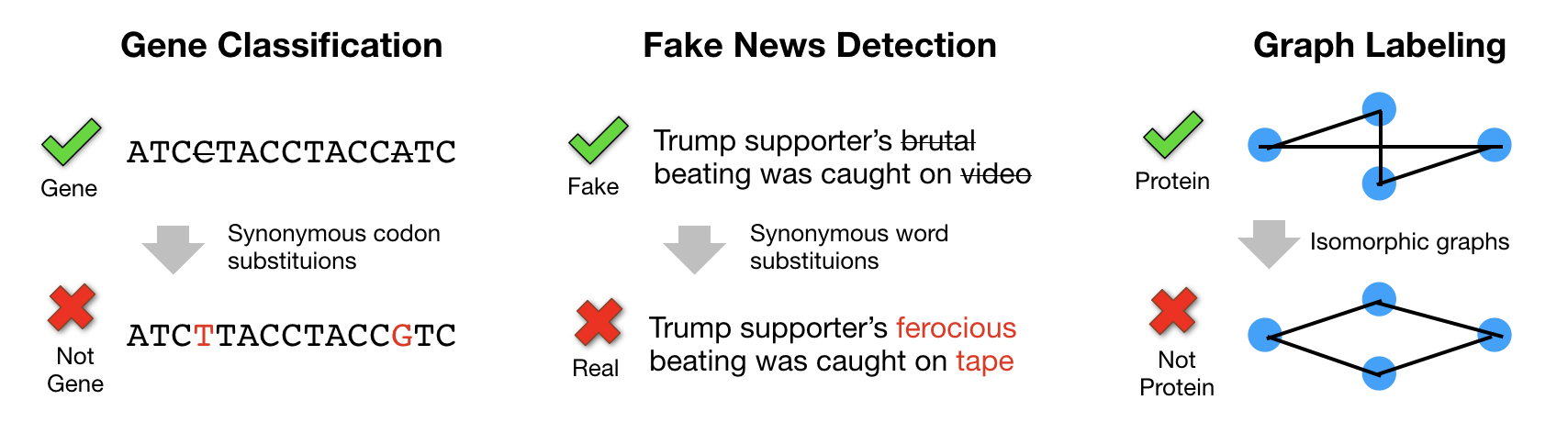}
\end{center}
\vspace{-5mm}
\caption{We study adversarial examples in domains in which inputs are discrete, such as genomics, language, or graphs. Synonymous adversarial examples are constructed for a discrete input by replacing individual symbols without changing the overall meaning of the input, for example by introducing synonymous nucleotide substitutions, synonymous words, or switching to an isomorphic graph. 
}
\end{figure*}

\section{Background}
We study classification problems, in which the goal is to learn a mapping $f : \mathcal{X} \to \mathcal{Y}$ from a discrete input $x \in \mathcal{X}$ to a target label $y \in \mathcal{Y}$, which lies in some finite set of $K$ classes $\mathcal{Y} = \{y_1, y_2, ..., y_K\}$. 
The classifier $f$ associates a score $f_{y_k}(x)$ to each class $y_k$ and outputs the class with the highest score.


\paragraph{Adversarial Examples for Continuous Inputs.}
In settings such as image classification, \citep{szegedy2014intriguing}, 
we say that $x' \in \mathbb{R}^d$ is an adversarial perturbation of $x\in \mathbb{R}^d$ targeting class $y'$ 
if
\begin{align}
& f(x') = y' \;\; \text{and} \;\;  \lVert x - x' \rVert \le \epsilon. \label{eq:image_ae}
\end{align}
The norm $\lVert \cdot \rVert$ captures the notion of an imperceptible perturbation. 
In this context, we refer to $x'$ as an adversarial example for $f$.  Adversarial examples can be obtained by solving an optimization problem of the form
\begin{align}
& \max_{x'} J(x') \;\; \text{s.t.} \;\;  \lVert x - x' \rVert \le \epsilon, \label{eq:image_ae_opt}
\end{align}
where the objective $J(x')$ measures the extent to which $x'$ is adversarial and may be a function of a target class $y' \neq y$. For example, we may take  $J(x') = f_{y'}(x')$ for some target class $y'$.
Algorithms for solving the above objective include the Fast Gradient Sign method or iterative methods based on constrained gradient descent \citep{goodfellow2014explaining,papernot2016limitations}.

\paragraph{Discrete-Input Adversarial Examples.}
Multiple authors proposed adversarial examples for text classification \citep{alzantot2018generating,ebrahimi2018hotflip}, initiating an extensive line of work \citep{gao2018black,jia2019certified}, including recent work on BERT \citep{li2020bert}.
Adversarial attacks also affect graphs \citep{zugner2018adversarial,dai2018adversarial}, usually by modifying embeddings \citep{bojchevski2019adversarial} or by adding/deleting nodes \citep{wang2019attacking}. Each line of work is independent and involves specialized definitions and methods hand-crafted for its domain.



\section{Synonymous Adversarial Examples}

The motivating goal of this paper is to extend adversarial examples to a new domain: computational biology. However, the area of computational biology is broad, featuring many diverse subfields.
Rather than crafting adversarial attacks for each subfield, we approach this problem via a principled framework of synonymous adversarial examples that applies across domains. 


%

\paragraph{Definition.}

Synonymous adversarial examples are constructed by perturbing inputs within a natural and domain-specific equivalence class of inputs such that they are misclassified.
More formally, given a classifier $f$, we say that $x'$ is an adversarial synonym of $x$ targeting class $y'$ if
\begin{align}
& f(x') = y' \;\; \text{and} \;\;  d(x, x') \le \gamma,
\end{align}
for some domain-specific user-specified distance function $d : \mathcal{X} \times \mathcal{X} \to \mathbb{R}^L_+$ and vector of bounds $\gamma \in \mathbb{R}^L$. The distance  is small when inputs $x_1,x_2$ are in some sense equivalent.

Such similarity functions can be naturally defined in many domains. For example, if $x, x'$ are graphs, we may say that $d(x, x') = 0$ if $x, x'$ are isomorphic to each other. If $x_1$ and $x_2$ are words, the function $d$ may gauge the extent to which they are synonyms.
In the context of image classification, we recover the original notion of adversarial examples by taking $d$ to be an $\ell_2$ or $\ell_\infty$ norm constraint. 

Our definition differs from other types of adversarial inputs explored in the literature. These include obfuscated examples \citep{carlini2016hidden}---in which the input appears as white noise but triggers unwanted behavior (e.g., audio that turns on a smartphone)---and concatenative examples \citep{jia2017adversarial}---in which the input is appended with a distracting sequence that contains irrelevant information.

\subsection{Constructing Adversarial Examples}

As in the continuous setting, we propose constructing synonymous adversarial examples by solving the optimization problem
\begin{align}
& \max_{x'} J(x') \;\; \text{s.t.} \;\; d(x, x') \le \gamma, \label{eq:text_ae_opt}
\end{align}
in which the objective $J(x')$ measures the extent to which $x'$ is adversarial (e.g., either minimizes the probability of the true label $y$ or maximizes the probability of a target label $y'$ different to $y$).

However, computing adversarial examples over discrete structures requires a specialized algorithm, as  gradient-based methods are no longer applicable.
We propose a beam search based greedy algorithm inspired by the adversarial text classification literature (Algorithm \ref{alg:opt}).
We show that it works well across domains (including new ones), while still retaining simplicity.


\begin{algorithm}[htb]

\KwData{ Datapoint $x$, objective function $J$, termination threshold $\tau$, beam size $b$, parameters $\gamma, \delta$. }
Initialize beam search at unperturbed data point:\\
$B \leftarrow \{x\}$\\
\While{\\
$\forall x' \in B,~ J(x') < \tau$
and fraction of symbols replaced $\le \delta$}{	
	Create a working set $W = \emptyset$ \;\\
	\For {each sequence $x$ in $B$}{	
    	\For {each symbol $w$ in $x$} {
        	\For {each symbol $\bar w$} {
              substitute $w$ with $\bar w$ to get $\bar x$; if $d(\bar x) \leq \gamma$, then $W \leftarrow W \cup  \{\bar x\} $\;
          }
      }
    }
    \eIf{$W = \emptyset$}{
   \textbf{break} \;
   }{
   	$B \leftarrow$ \{ top $b$ $\bar x \in W$ with highest values of $J(\bar x)$\}
  	}
}
\KwRet{$\arg\max_{\bar x \in B} J(\bar x)$}\;
\caption{Finding Adversarial Examples \label{alg:opt}}
\end{algorithm}

\paragraph{Algorithm Inputs.} We assume that an input $x$ can be represented a {\em sequence} of $n \geq 1$ discrete {\em symbols} or {\em tokens} and denoted by $w_i$ for $i=1, 2, \ldots n$. Algorithm \ref{alg:opt} seeks $x$ that maximize the objective $J$ using a variant of beam search. At each step, it considers all valid one-symbol changes to a sequence (which satisfy our constraints) and chooses the one that improves the objective the most. 

For example, in a genomics setting, inputs may correspond to DNA sequences of genes and symbols are individual nucleotides. Algorithm 1 can be used to replace to nucleotides with synonymous mutations, which do not change the underlying sequence of codons, and hence do not change the protein encoded by the gene. It turns out that such substitutions may nonetheless break gene classification algorithms.


\section{Empirical Analysis}\label{sec:experiments}

Next, we validate the utility and the generality of our approach empirically in new domains (in biology), and in a well-studied domain, text classification.


\subsection{Experimental Tasks}
We study adversarial examples on tasks in genomics, natural language processing, and graph classification.
Our genomics task is to predict whether a given sequence of nucleotides is a valid coding sequence for a particular species. In our experiments, we use mouse exons \citep{biomart} as positive examples. For negative examples, we use true mouse exons, with 50\% of the nucleotides randomly changed. Further details are provided in Appendix B.

We also examine three natural language classification tasks: sentiment analysis on the Yelp Review Polarity dataset \citep{zhang2015character}, spam classification on the TREC 2007 Public Spam Corpus, and fake news classification on the News Dataset \citep{mcintire2017news}. Further details are provided in Appendix C.

Finally, we perform graph classification on several protein datasets, and we define synonymous mutations as isomorphic changes to the shape of the graph, i.e.~a permutation of the rows of the adjacency matrix. This also corresponds to a relabelling of the vertices.

\subsection{Experimental Models}

\paragraph{Naive Bayes} This linear model has a long history in text classification and it is still popular for its simplicity. We convert each document into a bag-of-words representation, and following \cite{wang2012baselines}, we binarize the word features and use a multinomial model for classification. 
\paragraph{Long short-term memory} 
We built a single-layer LSTM with 512 hidden units as in \cite{zhang2015character}. The input to the LSTM is first transformed to a 300-dimensional vector using pretrained \texttt{word2vec} embeddings \cite{mikolov2013efficient}. We then average the outputs of the LSTM at each timestep to obtain a feature vector for a final logistic regression to predict the sentiment. 
\paragraph{Convolutional neural networks} 
We train a world-level CNN with a uniform filter size of 3 in each convolutional feature map; all other settings are identical to those of \cite{kim2014convolutional}.
We also implement the 9-layer character-level network of \cite{conneau2016very}, which includes 4 stages. Each stage has 2 convolutional layers with batch normalization and 1 max-pooling layer; convolutional and pooling layers have strides of 1 and 2, respectively and filters of size 3. We start with 64 feature maps, and double the amount after each pooling step, concluding with k-max pooling layer with $k=8$. The resulting activations in $\mathbb{R}^{4096}$ are fed to 3 fully connected layers.

\subsection{Genomic Sequence Classification}
\label{sec:genetics}

Recall that our task is to predict whether  a sequence of nucleotides is a valid coding region in the mouse genome.
Coding sequences admit a natural similarity metric: when a nucleotide substitution is synonymous, it does not affect the encoded amino acids produced or the resulting protein. Each codon has a set of synonymous substitutions (i.e., 0-6 other codons that would result in the same protein, if substituted in its place); we use this to define our distance function. Thus, for each codon we simply consider up to $N=6$ candidates for it to be replaced with.
We are conservative in our definition of equivalence to ensure validity, but more general approaches also exist.

\begin{wraptable}{right}{3.5cm}
\centering
\begin{tabularx}{3.5cm}{r | c  c}
\toprule
Model & Clean & Adv\\ \midrule
LSTM & 94.6 & 8.9  \\
CNN & 90.8 & 8.0 \\
NB & 88.5 & 8.7\\
\end{tabularx}
\caption{Accuracies for the gene classification  task}
\label{tab:genetics}
\end{wraptable}

More, formally we define the domain-specific distance $d(x_1, x_2)$ between exons $x_1, x_2$ to be zero if they are related by synonymous nucleotide substitutions, and infinity otherwise, and set $\gamma=0$. Additional hyperparameter details are provided in Appendix B.

We use the Ensembl Biomart \cite{biomart} database to download the set of exons for the mouse species (GRCm8.p6). For each exon, we obtain its valid reading frame, thus understanding how the nucleotides are broken up into codons (nucleotide subsequences of length 3). Our training set is of size 100,000 and our test set is of size 10,000. In our experiments, we work with exons of length up to 400 ($>$90\% of all exons), thus each exon is a sequence of up to 400 symbols, each being one of A,G,T,C.
We report the results of applying our algorithm in Table \ref{tab:genetics}. In the LSTM and CNN models, we substituted $9-10\%$ of nucleotides in an exon on average; for the Naive Bayes model, we substituted around $5\%$, suggesting it is easier to fool.

\subsection{Natural Language Classification}
\label{sec:nlpExps}

\begin{table}
\centering
\begin{tabular}{@{}ccrrrr@{}}
\toprule
 \multicolumn{2}{c}{Data}  & NB & LSTM & WCNN & VDCNN \\ \midrule
\multirow{3}{*}{Trec07p}    &   CLN & 97.1\% & 99.1\%    &  99.7\%     & ---          \\
						 & RND  & 97.7\% & 	98.6\%	& 	99.6\%	& ---	 	 \\
                               &   ADV   & 15.1\% & 39.8\%    &  64.5\%     & ---           \\ \midrule

\multirow{3}{*}{Yelp}       &   CLN & 88.01\% &   94.70\%  &    94.18\%   &      94.88\%     \\
							& RND & 86.8\% & 94.5\%	& 94.7\%	& 93.1\%	 	 \\
                               &   ADV  & 44.50\% & 33.19\%  &  34.48\%     &  49.80\%    \\ \midrule         
\multirow{3}{*}{News} &   CLN & 90.70\% &  93.10\%   &   96.00\%    &     93.40\%      \\ 
						& RND  & 84.0\% &	94.6\%& 	93.3\% &	92.7\% 	\\
                               &   ADV  & 50.95\% & 30.03\%    &  79.58\%     &  11.0\% \\ \bottomrule
\end{tabular}
\caption{Classifier accuracy on each dataset. CLN, RND, and ADV stand for clean, randomly corrupted, and adversarially corrupted inputs, respectively.}
\label{tab:results}
\end{table}

\paragraph{Similarity Distance.}

In a natural language context, we would like the synonymous examples $x'$ to retain the same meaning as the original $x$. 
%
We capture the above intution using $L=2$ constraints.

\textit{Semantic similarity:}  We capture the meaning of an utterance using an embedding in which similar sentences are close to each other. Formally, our constraint is
$
	\lVert {v} - {v}' \rVert_2 < \gamma_1 \label{semantics-constraint}
$
where $v$ and $v'$ are thought vectors associated with $x$ and $x'$; sentence embeddings are averages of the word embeddings of \citep{mrksic2016naacl}  which represent synonymity, not relatedness.

\textit{Syntactic similarity:} 
Our syntactic constraint requires that a language model $P : \mathcal{X} \to [0,1]$ trained on the same dataset assigns similar probabilities to the perturbed and the original example:
$
	| \log P( x' ) - \log P( x ) | < \gamma_2  \label{syntax-constraint}
$.
The language model captures the extent to which x “looks like” a spam message or a movie review. 

\paragraph{Experimental Results.}
Table \ref{tab:results} shows the accuracy of each classification model on the three clean datasets as well as on adversarial inputs generated using Algorithm \ref{alg:opt}. We also report accuracies on randomly perturbed examples. 
Details on hyperparameters and the implementation used are provided in Appendix C.

The average fraction of words substituted to be close to 10\%; the threshold of $\delta$ is just an early stopping criterion which is not often reached. Empirically, we observe that the adversarial examples we construct tend to closely retain the intent of the original message.

We verify the quality and the coherence of our examples via human experiments on Amazon Mechanical Turk. As Table \ref{tab:human_accqual} shows, human evaluators achieve similar levels of success at classifying both the original and adversarial examples; they also assign both similar scores when asked to rate them (from 1-5) based on overall ``writing quality''. 


\begin{framed}
\small{\textbf{Task:} Spam filtering. \textbf{Classifier:} LSTM. \textbf{Original label:} 100\% Spam. \textbf{New label:} 89\% Non-Spam.}
\vspace{1mm}

\small{{\bf Text}: your \sout{application} \substi{petition} has been \sout{accepted} \substi{recognized} thank you for your \sout{loan} \substi{borrower} \sout{request} \substi{petition} , which we recieved yesterday , your \sout{refinance} \substi{subprime} \sout{application} \substi{petition} has been \sout{accepted} \substi{recognized} good credit or not , we are ready to give you a \$ oov loan , after further review , our lenders have established the lowest monthly payments . approval process will take only 1 minute . please visit the confirmation link below and fill-out our web-form}
\vspace{1mm}

%
%

\hrule
\vspace{1mm}

\small{\textbf{Task:} Gene classification. \textbf{Classifier:} CNN. \textbf{Original label:} 95\% Gene. \textbf{New label:} 7\% Gene}
\vspace{1mm}

\small{{\bf DNA Sequence}: 
...ATC\sout{C}\substi{T}CTCCTCAT\sout{G}\substi{A}CCACCT...}
\end{framed}
\label{adv-sample}

\subsection{Graph Classification}

We study graph classification models over three public bioinformatics datasets: ENZYMES, PROTEINS, and NCI1, which contain between 600 and 4000 graphs of an average size between 30 and 40. We set the node features to be a vector of ones for every node, and we only use the graph structure captured within the adjacency matrix to perform classification. On each dataset, we train a graph convolutional neural network with one hidden layer of size 128.

\begin{wraptable}{right}{3.5cm}
\centering
\begin{tabularx}{3.5cm}{r | c  c}
\toprule
Model & Clean & Adv\\ \midrule
{\tiny ENZYMES} & 40.6 & 11.9  \\
{\tiny PROTEINS} & 65.4 & 30.1 \\
{\tiny NCI1} & 68.9 & 21.2\\
\end{tabularx}
\caption{Accuracies on graph classification}
\label{tab:graphs}
\end{wraptable}

Recall that two graphs are said to be isomorphic if there exists a bijection (i.e., a relabelling) of the vertices that makes the two graphs identical. Clearly, isomorphic graphs are considered to be equivalent. Formally, we   define the domain-specific distance $d(x_1,x_2)$ between graphs $x_1,x_2$, to   be   zero   if   they are  isomorphic, and infinity  otherwise,   and  we set $\gamma=  0$. 
We apply the same greedy algorithm as in other tasks, except we now search in the space of row permutations among the input graphs.

We report the results of applying our algorithm in Table \ref{tab:graphs}. As in earlier tasks, the adversarial accuracy is significantly below the accuracy of the algorithm on clean data. We observe that across the three datasets, our algorithm permuted 15\% of rows on on average.

\section{Features of Adversarial Examples}\label{sec:features}

\subsection{Transferability Analysis}
\begin{table}
\hspace{-3mm}
\begin{tabular}{lccccc}
\toprule
		& {\small CLN} & {\small NB}	& {\small LSTM}	& {\small WCNN}	& {\small VDCNN} \\ \midrule
{\tiny NB}		& 88.01 & 44.50\% & 80.76\% & 79.90\% & 88.03\% \\
{\tiny LSTM}	& 94.70 & 81.12\% & 33.19\% & 85.17\% & 90.99\% \\
{\tiny WCNN}	& 94.18 & 78.17\% & 82.45\% & 34.48\% & 89.96\% \\
{\tiny VDCNN}	& 94.88 & 81.94\% & 88.50\% & 88.08\% & 49.80\% \\ \bottomrule
\end{tabular}
\caption{Transferability of adversarial examples on the Yelp task. Cell $(i,j)$ shows the accuracy of adversarial samples generated for model $i$ evaluated on model $j$.}
\label{cross-model}
\end{table}

We examine whether synonymous adversarial examples transfer between models (as they often do between image classifiers), focusing on the Yelp dataset. As seen in Table \ref{cross-model}, there is a moderate degree of transferability between models, with accuracy on adversarial examples generated for another model being lower than accuracy on the clean (and randomly perturbed) dataset

\subsection{Defending Against Adversarial Attacks}
\label{sec:miyatoSection}

There have been several proposed methods of securing classification algorithms against adversarial examples \citep{advDefense1,advDefense2,ensembleAdvDefense}. In particular, ~\citet{miyato2017virtual} introduces the idea of adversarial training for natural language classification tasks. They perform small adversarial perturbations of word embeddings during the training phase, in an attempt to learn a more robust network. To evaluate our adversarial attack algorithm against this defense, we train an LSTM model for the gene classification task, using the adversarial training method.

\begin{wraptable}{right}{4.5cm}
\centering
\begin{tabular}{r | c  c}
\toprule
& Vanilla & Hardened\\ \midrule
CLN & 94.6 & 94.7  \\
ADV & 8.9& 33.9 \\
\end{tabular}
\caption{Accuracy comparison against a model trained with adversarial perturbations on gene classification.}
\label{tab:miyatoComparison}
\end{wraptable}
In Table \ref{tab:miyatoComparison}, we report the performance of the hardened model on the clean dataset and adversarial examples generated for it. We observe an increase in accuracy on the perturbed example, indicating that the model has improved robustness against adversarial examples. However, our algorithm still results in a significant reduction in performance on the model, thus showing that our attack is nevertheless quite potent. 

\subsection{Comparison vs.~Other Attack Methods}
In the context of text classification, we compared against earlier specialized methods to validate the effectiveness of our approach, which we derived in a domain agnostic way. We removed the language model constraint and experimented with beam search size and maximum word substitutions, essentially reproducing the approach of \citet{ebrahimi2018hotflip}. We also removed the semantic constraint, which reduces to the unconstrained method of \citet{papernot}.
Finally, we tested a version that removes semantic constraints, but keeps the syntactic constraint; these comparisons also yield an ablation analysis of our model.

In the context of text classification, a fully quantitative comparison is challenging: all methods generate strong adversarial perturbations, but their similarity to the original input is only measurable qualitatively \citep{morris2020textattack}, and we relied on human analysis. We determined that unconstrained adversarial examples \citep{papernot} were perceived as being of poor quality; adding semantic or syntactic constraints \citep{ebrahimi2018hotflip} yielded significantly improved results, but with certain error patterns. We report these errors and include samples from the above methods in Appendix D. We also compare against the approach of  \citet{naturalAEs}---in which a generative model is used to construct adversarial inputs---in Appendix E.

\begin{wraptable}{right}{4cm}
\hspace{-4mm}
\begin{tabular}{r | c  c}
\toprule
& GreedyP & Beam \\ \midrule
Gene & 12.3 & 8.9  \\
Yelp & 10.9 & 5.9 \\
\end{tabular}
\caption{Accuracy of an LSTM on the Yelp and gene classification tasks on adversarial examples generated with prioritized greedy search \& with our full beam search.}
\label{tab:search_alg}
\end{wraptable}
Next, we experimented with the search algorithm for finding adversarial examples given our set of constraints. We compared against simple greedy word substitution (in text classification, this yields essentially the method of \citet{jin2020bert}), as well as word substitution prioritized by the importance of the word on the target label (``Greedy-P"). 
We tested this on LSTM models for gene classification and Yelp sentiment analysis (Table \ref{tab:search_alg}); while even simple methods yield good attacks, our beam search method further improves performance.
A further refinement of our work is the genetic search method of \citet{alzantot2018generating} (the constraints and the rest of their method are mostly the same as ours) and its fast version by \citet{jia2019certified}. This approach is more advanced and requires training language models in new domains; we defer adding these optimizations to future work.

\section{Towards Explaining the Existence of Synonymous Adversarial Examples}




We attribute the existence of adversarial perturbations to spurious token correlations, a statistical phenomenon that is specific to discrete spaces and that is caused by a mismatch between the dimensionality of a problem (e.g., the size of a natural language vocabulary) and the size of the dataset. 

\subsection{Intuition}
Given a finite dataset, as we increase the size of the vocabulary, some discrete tokens will by pure chance appear to be associated more frequently with one of the classes, even though in expectation, their association with all the classes should be the same.
For example, in a sentiment classification dataset, the word ``definitely" may by pure chance occur more frequently in positive reviews. Even though the probability of this happening is small, given a large enough vocabulary, we will eventually observe a number of such words. We refer to them as spurious tokens.

Such spurious token naturally lead to adversarial examples, even in such simple models as Naive Bayes.
In the above example, if the synonym ``certainly" is not a spurious token, we can replace all of its occurrences with ``definitely", which will contribute towards adversarially flipping the class without changing the meaning of the overall input text.

Additionally, certain synonymous tokens may be correlated with a certain class for reasons beyond simple statistical noise. Certain tokens may truly have different conditional class probabilities because of multiple meanings (e.g., awfully great deal and awfully bad) or because of hidden patterns that we don’t understand (e.g., dogs are truly mentioned more often in spam than cats). In both cases, these tokens can also be used to construct adversarial examples.

\paragraph{Continuous Perturbations.}
Once an input $x$ has been embedded into a continuous representation, the remainder of the neural network becomes susceptible to the same types of attacks that affect continuous input classifiers. 
This intuition is substantiated by our experiments in Section \ref{sec:features}: when adversarial training has been performed on the embedding layer as proposed by \citet{miyato2017virtual}, the performance of our algorithm goes down, suggesting these embedding-based vulnerabilities become harder to exploit.

\subsection{Theoretical Analysis}

We expand upon this intuition via a formal analysis. The full version with the proofs is in Appendix A.

\paragraph{Model} Consider a dataset of size $D$ over a vocabulary of size $V$. Each data point $d \in [D]$ is comprised of $L$ discrete tokens. 
We consider a data-generating process in which there are two classes labeled by $Y \in \{0,1\}$. Each of the $D$ data points is generated by a process in which we first sample a class $Y$ uniformly at random; $L$ tokens comprising the datum are then sampled from a multinomial distribution with parameters $p_{ly}$ for $l \in [L]$ and $y \in \{0,1\}$.
We assume that there are two types of tokens. Informative tokens are truly indicative of the class $y$ and $p_{l0} \neq p_{l1}$. Most tokens are considered to be uninformative, meaning that $p_{l0} = p_{l1}$. 

We work with a trained classifier $p_\theta(Y=y|x)$ that models the probability of class $y$ given inputs $x$.
Without loss of generality, we assume the attacker wants to change the prediction of $p_\theta$ from 0 to 1.

Let $w_{ly} = \log(\hat p_{ly})$ be the empirical log-probability of token $l$ in class $y$. 
We define a token to be spurious if the empirical log-ratio of probabilities $\delta_l = w_{l1}-w_{l0}$ exceeds a threshold $\gamma > 0$.

\subsection{Spuriously Correlated Tokens Exist}

We start by proving and quantifying the existence of spuriously correlated tokens within this idealized framework.
The following proposition makes precise our intuition about the existence of spuriously correlated tokens.

\begin{proposition}
Suppose that the occurrence probability $p_{l0} = p_{l1}$ of each of V uninformative token $l$ at most $p$. 
The expected number of spurious tokens is given by 
$$ V \left(1 - \Phi(\gamma / \sigma(p))\right),$$
where $\sigma^2(p) = \frac{1}{D}\left(\frac{1}{p} - 1\right)$ and $\Phi$ is a CDF of a standard Gaussian.
\end{proposition}
\begin{proof}(Sketch) 
The log-probability of a token $l$ for a class $y$ is given by
$ w_{ly} = \log(\hat p_{ly}) = \log \frac{D_{ly}}{D_y}. $
Recalling the assumption on the data generating process, we note that each $D_{ly}$ is distributed as a Binomial random variable $\mathrm{Bin}(D_y, p_{ly})$. 

We use the Lemma of Katz et al. in Appendix A to argue that $\delta_l = \log \frac{D_{l1}/D_1}{D_{l0}/D_0}$ is approximately normal with mean $\log \frac{p_{l1}}{p_{l0}} = 0$ and variance $\sigma^2_l = \frac{1}{D_0}(\frac{1}{p_{l0}} - 1) + \frac{1}{D_1}(\frac{1}{p_{l1}} - 1)$. 
\vk{Can you add more comments on why? Also, did you get the chance to double check if the assumptions of that paper hold?}\en{OK. added the proof of the lemma}
Assuming for simplicity that total number of words in both classes match, i.e. $D_0 = D_1 = \frac 12 D$, and since $p_{l0} = p_{l1} \leq p$, the variance $\sigma^2_l$ is bounded below by 
$
\sigma^2(p) = \frac{1}{D}\left(\frac{1}{p} - 1\right).
$

Considering the probability of a token being spurious, we get 
    $$
    \mathbb{P}[\delta_l > \gamma] = 1 - \Phi(\gamma / \sigma_l) \geq 1 - \Phi(\gamma / \sigma(p)).
    $$ 
Thus, the expectation of total number of spurious tokens is bounded below by the expectation of a Binomial random variable with the total number of trials $V$ and the probability of success $1 - \Phi(\gamma / \sigma(p))$. Hence, the expected number of spurious tokens is at least $V \left(1 - \Phi(\gamma / \sigma(p))\right)$.
\end{proof}

In order to analyze the effects of spurious tokens, we make the assumption that their presence influences the predictions of the model.

\begin{assumption}
Let $t$ be a spurious token. Appending $t$ to input $x$ or replacing a token in $x$ with $t$ increases the score of class $y=1$ by at least $\gamma>0$.
\end{assumption}

In other words, we assume that the model $p_\theta$ has been correctly trained and captures the signal in the data. 
This assumption is needed for non-convex models such as neural networks for which training is not guaranteed to succeed.

Note that this assumption is true by construction for Naive Bayes (and extends to other linear models). The value $w_{ly}$ is the weight learned by NB for token $l$ and class $y$. The presence of token $l$ increases the score of class $1$ by $\delta_l = w_{l1}-w_{l0} > \gamma$. For deep learning models, we demonstrate empirically that most spurious tokens influence predictions in Section \ref{sec:experiments} and we model their strength by $\gamma>0$. 

\subsection{Concatenative Adversarial Examples}
A direct consequence of the existence of spuriously correlated tokens is that given a sufficiently large vocabulary, we may construct examples by appending a small number of such tokens to a data point. This explains the existence of concatenative adversarial examples, which have been recently studied in the context of reading comprehension
\citep{jia2017adversarial}. 

The proposition below gives a formal expression for the number of extra tokens that need to be concatenated to form such an adversarial example. Assume that for an informative token $l$, we have
$ \eta < p_{ly} / p_{ly'}, $
where $y$ is the true class of $l$ and $y'$ is the other class.
The parameter $\eta > 0$ serves as a measure of signal strength, with larger values of $\eta1$ indicating that each informative token carries more signal, making is easier to learn a model for distinguishing between the two classes.

\begin{proposition}
Suppose that the occurrence probability $p_{l0} = p_{l1}$ of each of $V$ uninformative token $l$ at most $p$. Consider a data point of $L$ tokens and suppose that it contains at most $r \cdot L$ informative tokens with signal strength at most $\eta$. Then, with probability at least $1-\rho$, the number of adversarial insertions that is needed to be made to change the predicted class is at most
$$\left\lceil \frac{rL\eta}{\sigma(p) \Phi^{-1}(\rho^{\frac{1}{V}})}\right\rceil,$$
where $\sigma^2(p) = \frac{1}{D}\left(\frac{1}{p} - 1\right)$ and $\Phi$ is a Normal CDF.
\end{proposition}

\subsection{Synonymous Adversarial Examples}

Next, we analyze the susceptibility of models to adversarial examples in our idealized framework.
We consider an adversarial attack model in which each token has $S$ possible synonyms, assumed to be non-informative. Our adversarial attack model on a data point allows $T$ substitutions of a token to one of its $S$ admissible synonyms. We further assume that within each data point $d$, at most a fraction $0 < r < 1$ of tokens are informative, for a total of $r \cdot L$ tokens that are indicative of the true label.

Assume again that for an informative token $l$, we have
$ \eta < p_{ly} / p_{ly'}, $
where $y$ is the true class of $l$ and $y'$ is the other class.

\begin{proposition}
Given a budget $T$ of adversarial changes, we can swap $T$ non-informative tokens to their synonyms so the predicted class is changed with probability at least
$$\sum_{t=T}^{(1-r)L} \binom{(1-r)L}{t} \left(1-\phi\right)^t \phi^{(1-r)L-t},$$
where $\phi = \left(\Phi\left(\frac{rL\eta}{\sqrt{2}T\sigma(p)}\right)\right)^S$, $\sigma^2(p) = \frac{1}{D}\left(\frac{1}{p} - 1\right)$ and $\Phi$ is a CDF of a standard Gaussian. \qedsymbol
\end{proposition}

The above proposition establishes conditions when perturbations exist as a function of dataset size and other parameters.

\section{Previous Work and Discussion}

\paragraph{Adversarial Attacks on Text and Graphs.}
Multiple authors proposed adversarial examples for text classification problems \citep{alzantot2018generating,ebrahimi2018hotflip}, supporting an extensive line of work, including character-level models \citep{gao2018black}, word deletion \citep{feng2018pathologies}, fast defenses \citep{jia2019certified}), and attacks on BERT \citep{jin2020bert,garg2020bae,li2020bert}.
\citet{morris2020textattack} provides a unified framework for these methods.
Adversarial attacks also affect graphs \citep{zugner2018adversarial,dai2018adversarial,wu2019adversarial}, usually via black-box modifications \citep{chang2020restricted} to node or edge embeddings \citep{bojchevski2019adversarial}, via adding or removing nodes \citep{wang2019attacking}, or on graph matchings \citep{zhang2020adversarial}
We include a discussion of these methods and additional works in Appendix E.

\paragraph{Domain-Agnostic Methods.} Our work seeks to generalize across domains, although some components (e.g., constraints) still need to be crafted from prior knowledge. However, the resulting framework yields general theoretical insights, and provides strong guidance on how to create and study adversarial examples in new domains with less effort.

Our search algorithm is also general. Note that it is similar to search methods in NLP attacks. The sophistication of our search method is not a primary contribution; rather, we show that it works well across domains (including new ones), while still being simple.

\section{Conclusion}
We define synonymous adversarial examples across discrete domains, and we use a simple beam search-like algorithm for generating them in both well-studied and new domains, such as biology.
%
We offer insights into why these vulnerabilities exist and suggest a way to improve the robustness of classification algorithms via adversarial training.

\newpage
\bibliographystyle{icml2020}
\bibliography{paper}
\newpage
\appendix
\section{Theoretical Analysis}

In this section, we provide a theoretical analysis of discrete-space adversarial examples. 

\subsection{Model}

We start with the following idealized model of the data generating process.

Consider a dataset of size $D$ over a vocabulary of size $V$. Each data point $d \in [D]$ is comprised of $L$ discrete tokens. We are interested in analyzing and explaining the susceptibility of algorithms to adversarial examples over this dataset.

\subsubsection{Multinomial Naive Bayes}

Recall that given a vector $x$ of size $V$ indicating the number of times each word in the vocabulary occurs in a data point, a Naive Bayes classifier model the probability of assigning a class $y \in Y$ to the data point using Bayes' theorem:
\begin{equation}
    \label{eq:bayes}
    p(y | x) = \frac{p(y) p(x | y)}{p(x)}.
\end{equation}

The prior distribution of $y$ is estimated as a proportion of each class in the dataset: $p(y) = \frac{1}{D} \sum_{i=1}^D [y_i = y]$. There are several possible choices of distribution of $p(x | y)$; in this analysis, we focus on Multinomial Naive Bayes, which assumes that the likelihood $p(x | y)$ is given by a multinomial distribution with parameters $(\hat p_{1y}, \dots, \hat p_{Ly})$ estimated from data: $\hat p_{ly} = \frac{D_{ly}}{D_y}$, where $D_{ly}$ indicates the number of times the word $l$ occurs in data points with class $y$, $D_y$ --- the total number of words in data points with class $y$. 

The algorithm assigns the class with the highest probability to a data point. Taking the log of both sides of \ref{eq:bayes}, we note that the decision rule is linear with respect to the input $x$:
$$
    \log p(y | x) \propto \log p(y) + \sum_{l=1}^L x_l \log \hat p_{ly} = b_y + \langle w_y, x \rangle,
$$
where $w_y = (\log \hat p_{1y}, \dots, \log \hat p_{Ly})$.

\subsubsection{Setup}

We consider a data-generating process in which there are two classes labeled by $Y \in \{0,1\}$. Each of the $D$ data points is generated by a process in which we first sample a class $Y$ uniformly at random; $L$ tokens comprising the datum are then sampled from a multinomial distribution with parameters $p_{ly}$ for $l \in [L]$ and $y \in \{0,1\}$.

We assume that there are two types of tokens. Informative tokens are truly indicative of the class $y$ and $p_{l0} \neq p_{l1}$. Most tokens are considered to be uninformative, meaning that $p_{l0} = p_{l1}$. In order to simplify our analysis, we assume that for an informative token $l$, we have
\begin{align*}
    \eta_0 < \log \frac{p_{l0}}{p_{l1}} & & \eta_1 < \log \frac{p_{l1}}{p_{l0}}
\end{align*}
for some parameters $\eta_0, \eta_1 \in \mathbb{R}$ that serve as a measure of the learning difficulty of the problem. Intuitively, larger values of $\eta_0, \eta_1$ indicate that each informative token carries more signal, making is easier to learn a model for distinguishing between the two classes.

We further assume that within each document $d$, at most a fraction $0 < r < 1$ of tokens are informative, for a total of $r \cdot L$ tokens that are indicative of the true label.

We consider an adversarial attack model in which each token has $S$ possible synonyms. For simplicity, we assume a sufficiently large vocabulary, such that the $S$ synonyms are all uninformative tokens. The adversarial attack model allows performing $T$ substitutions on a given data point; each substitution consists of changing the token to one of its $S$ admissible synonyms. \en{(Uncomment me) Further in this analysis, we will assume that the attacker is interested in changing the prediction from 0 to 1.}

\subsection{Results}

Next, we perform an analysis of the above attack model within the idealized framework. The goal of the analysis is to explain a specific source of vulnerability that affects discrete-input adversarial examples and that is distinct from the infinitesimal adversarial perturbations that are widely known to affect machine learning models over continuous inputs.

Specifically, we claim that adversarial examples deep learning models over discrete inputs can be attributed to two distinct types of perturbations. The first type occurs in embedding space and is analogous to infinitesimal perturbations in continuous-input models. The second type involves perturbing the discrete symbols directly and can be explained by a statistical process that is distinct from the ones that explain infinitesimal continuous perturbations.


\subsubsection{Intuition}

Intuitively, perturbations over discrete inputs arise because of a mismatch between the dimensionality of the problem (i.e., large vocabulary sizes) and the amount of data available. Given a large vocabulary of tokens that are distributed with equal probability among two classes, a sufficiently large number of them will occur disproportionately often in one class relative to the other. These tokens represent a {\em spurious} vocabulary, since given a limited dataset, they will appear to be spuriously correlated with one of the two classes.

Having such spuriously-correlated tokens enables the creation of adversarial examples. As in early work on reading comprehension, a small number of spurious tokens that are highly correlated with some class can be appended to the existing set of tokens in a data point \citep{jia2017adversarial}. In the context of synonymous adversarial examples, we expect there will be multiple tokens with synonyms that correlate with the opposite class. By swapping these tokens, we can again arrive at adversarial examples.

\subsubsection{Spurious Vocabulary}

In this section, we start by quantifying the number of uninformative tokens that will be spuriously correlated with a given class. We will first need the following lemma:

\begin{lemma}
\vk{this is still missing the assumptions and stuff like that}\en{OK. added $0<p,q<1$ and a rigorous proof. I hope that you will enjoy the proof but feel free to delete it if you find it unnecessary. If you decide to delete, uncomment the citation above} Let us take two Binomial random variables $X \sim \mathrm{Bin}(N, p)$, $Y \sim \mathrm{Bin}(M, q)$ with $0 < p, q < 1$. Then, $\log \frac{X/N}{Y/M}$ is distributed approximately as a gaussian with mean $$\log \frac pq$$ and variance $$\frac{1}{N}\left(\frac{1}{p} - 1\right) + \frac{1}{M}\left(\frac 1q - 1\right).$$
\end{lemma}

\begin{proof}
Let us note that $X$ is a sum of $N$ Bernoulli random variables and denote $\hat p = X / N$. Consider Taylor expansion of $\log \hat p$ centered at $p$:
\begin{align*}
    \log \hat p = \log p + \frac{1}{p} (\hat p - p) + \dots
\end{align*}
Let us rearrange the terms and scale them by $\sqrt{N}$:
\begin{align*}
    \sqrt{N} (\log \hat p - \log p) = \frac{\sqrt{N}}{p} (\hat p - p) + \dots
\end{align*}

Using the Central Limit Theorem and recalling that the variance of a Bernoulli distribution is given by $p(1-p)$, $\sqrt{N}(\hat p - p)$ converges to $\mathcal{N}(0, p(1-p))$. We further note that the higher-order terms of the expansion converge to zero.
Thus, $\sqrt{N} (\log \hat p - \log p) \to \mathcal{N}(0, \frac{1-p}{p})$ which implies that $\log \hat p = \log \frac{X}{N}$ is approximately $\mathcal{N}\left(\log p, \frac{1-p}{Np}\right).
$
Finally, we note that $\log \frac{X/N}{Y/M}$ is a difference of two terms converging to gaussians meaning that $\log \frac{X/N}{Y/M}$ is approximately distributed as
$$
     \mathcal{N}\left(\log \frac pq, \frac{1-p}{Np} + \frac{1-q}{Mq}\right).
$$

\end{proof}

Now we are ready to formulate the first result on the expected number of spurious tokens:

\begin{proposition}\vk{This is missing some statement about which class we are trying to flip. Spurious with respect to which class? Please add it to the proof, saying something like "without loss of generality, we look at flip from zero to one..."}\en{OK. now it is in subsection A.1.2 and in the proof}
Suppose that we have a vocabulary of $V_u > 0$ uninformative tokens, with the occurrence probability $p_{l0} = p_{l1}$ of each token $l$ at most $p$. Define a token to be spurious if the empirical log-ratio of the probabilities is greater than a threshold $\gamma$. Then the expected number of spurious tokens is at least 
$$ V_u \left(1 - \Phi(\gamma / \sigma(p))\right),$$
where $\sigma^2(p) = \frac{1}{D}\left(\frac{1}{p} - 1\right)$ and $\Phi$ is a CDF of a standard gaussian. \en{is it ok that here we are also making a statement about a bound on the expectation of the number of spurious tokens? P.S. the high probability statement on the number is also possible, but going to be a little bit ugly}

\end{proposition}

\begin{proof}


Without loss of generality, we assume that a classifier assigns class 0 to a data point and the attacker is interested in changing the prediction to 1. To simplify the expressions, we assume that classes are balanced, i.e. $b_0 = b_1$, and we can ignore the bias term of a Naive Bayes (NB).
For an NB classifier, the weight of a token $l$ for a class $y$ is given by
    $$ w_{ly} = \log(\hat p_{ly}) = \log \frac{D_{ly}}{D_y}. $$
Formally, we consider a token being spurious if $\delta_l = w_{l1} - w_{l0} > \gamma$.
Recalling the assumption on the data generating process, we note that each $D_{ly}$ is distributed as a Binomial random variable $\mathrm{Bin}(D_y, p_{ly})$. 

We use Lemma 1 to argue that $\delta_l = \log \frac{D_{l1}/D_1}{D_{l0}/D_0}$ is approximately normal with mean $\log \frac{p_{l1}}{p_{l0}} = 0$ and variance $\sigma^2_l = \frac{1}{D_0}(\frac{1}{p_{l0}} - 1) + \frac{1}{D_1}(\frac{1}{p_{l1}} - 1)$. 
\vk{Can you add more comments on why? Also, did you get the chance to double check if the assumptions of that paper hold?}\en{OK. added the proof of the lemma}
Assuming for simplicity that total number of words in both classes match, i.e. $D_0 = D_1 = \frac 12 D$, and since $p_{l0} = p_{l1} \leq p$, the variance $\sigma^2_l$ is bounded below by 
    $$
    \sigma^2(p) = \frac{1}{D}\left(\frac{1}{p} - 1\right).
    $$

Considering the probability of a token being spurious, we get 
    $$
    \mathbb{P}[\delta_l > \gamma] = 1 - \Phi(\gamma / \sigma_l) \geq 1 - \Phi(\gamma / \sigma(p)).
    $$ 
Thus, the expectation of total number of spurious tokens is bounded below by the expectation of a Binomial random variable with the total number of trials $V_u$ and the probability of success $1 - \Phi(\gamma / \sigma(p))$. Hence, the expected number of spurious tokens is at least $V_u \left(1 - \Phi(\gamma / \sigma(p))\right)$.
\end{proof}

Also, we can show how many spurious tokens can be appended to a data point to change its class. We will first need a lemma characterizing bounds of maxima of i.i.d. centered gaussians.

\begin{lemma}
Let $X_1, \dots, X_n$ be i.i.d. random variables sampled from $\mathcal{N}(0, \sigma^2)$. Then, with probability $1 - \rho$, the following holds:
$$\max_{i \in [n]} X_i > \sigma \Phi^{-1}\left(\rho^\frac{1}{n}\right).$$
\end{lemma}

\begin{proof}

Applying basic rules of probability, we get:
\begin{align*}
    \mathbb{P}[\max_{i \in [n]} X_i > x] 
    &= 1 - \mathbb{P}[\max_{i \in [n]} X_i \leq x] = \\
    &= 1 - \left(\mathbb{P}[X_1 \leq x]\right)^n = \\
    &= 1 - \left(\Phi\left(\frac{x}{\sigma}\right)\right)^n,
\end{align*}
where $\Phi$ is a CDF of a standard gaussian.
Denoting $\left(\Phi\left(\frac{x}{\sigma}\right)\right)^n = \rho$, we arrive to $x = \sigma \Phi^{-1}\left(\rho^\frac{1}{n}\right)$ and obtain the desired bound.

\end{proof}

\begin{proposition}
Suppose we have a vocabulary of $V_u > 0$ uninformative tokens, with the occurrence probability $p_{l0} = p_{l1}$ of each token $l$ at most $p$. Consider a data point of $L$ tokens and suppose that it contains at most $r \cdot L$ informative tokens with signal strength at most $\eta$. Then, with probability at least $1-\rho$, the number of adversarial insertions that is needed to be made to change the predicted class is at most
$$\left\lceil \frac{rL\eta}{\sigma(p) \Phi^{-1}(\rho^{\frac{1}{V_u}})}\right\rceil,$$
where $\sigma^2(p) = \frac{1}{D}\left(\frac{1}{p} - 1\right)$ and $\Phi$ is a CDF of a standard gaussian.
\end{proposition}

\begin{proof}

We call signal strength of a token $l$ the difference of scores $w_{l0} - w_{l1}$ the NB classifier assigns towards class 0. 
Under the assumptions of the proposition, the total signal strength of a data point is bounded by $r \cdot L \cdot \eta$.

Denoting the total number of insertions by $T$, we are interested in making $T$ insertions such that they result in changing the prediction of a classifier from 0 to 1. \vk{Need to define total signal strength a bit better. Just say that we call signal strength the score of a correctly classifier datapoint and it is bounded by that expression under our assumptions} \en{OK. added above}
Recalling the result from the proof of Proposition 1, $\delta_l = w_{l1} - w_{l0}$ is distributted approximately as a zero mean gaussian with variance at least $\sigma^2(p) = \frac{1}{D}\left(\frac{1}{p} - 1\right)$. 
Since we are interested in changing the class with minimum editions possible, we take tokens with maximum $\delta_l$ among the $V_u$ possible choices. Using Lemma 2, and noting that $\delta_l$ are zero mean gaussians with variance at least $\sigma^2(p)$, we have $$\mathbb{P}[\max_{l \in [V_u]} \delta_l > \sigma(p) \Phi^{-1}(\rho^{\frac{1}{V_u}})] \geq 1 - \rho.$$
    \vk{This seems to be missing a probability for this bound.. I'm a bit confused here...}\en{OK. changed to a bound on maximum of gaussians, not its expectation}
    
Since the number of insertions $T$ is much less compared to the size of uninformative vocabulary $V_u$, we can assume that $T$ tokens with biggest $\delta_l$ are also bounded below by $\sigma(p) \Phi^{-1}(\rho^{\frac{1}{V_u}})$ with probability at least $1 - \rho.$
Thus, if we take $$T = \left\lceil \frac{rL\eta}{\sigma(p) \Phi^{-1}(\rho^{\frac{1}{V_u}})}\right\rceil,$$
    we have
    $$\mathbb{P}[T \max_{l \in [V_u]} \delta_l > rL\eta] \geq 1 - \rho.$$

\end{proof}

\subsubsection{Susceptibility to Synonymous Adversarial Examples}

Next, we consider the setting of synonymous adversarial examples.

\begin{proposition}
Consider a data point of $L$ tokens and suppose that it contains at most $r \cdot L$ informative tokens with signal strength at most $\eta$. Suppose further that each of the rest $(1-r) \cdot L$ tokens is non-informative and has $S$ non-informative synonym tokens with the occurrence probability $p_{l0} = p_{l1}$ for each token $l$ or its synonym at most $p$. 
Then, given a budget $T$ of adversarial changes, we can swap $T$ non-informative tokens to their synonyms so the predicted class is changed with probability at least
$$\sum_{t=T}^{(1-r)L} \binom{(1-r)L}{t} \left(1-\phi\right)^t \phi^{(1-r)L-t},$$
where $\phi = \left(\Phi\left(\frac{rL\eta}{\sqrt{2}T\sigma(p)}\right)\right)^S$, $\sigma^2(p) = \frac{1}{D}\left(\frac{1}{p} - 1\right)$ and $\Phi$ is a CDF of a standard gaussian.
\vk{Can you please convert the combinatorial into powers, e.g. L choose t $\approx L^t$?}\en{it doesn't make much sense to convert binomial coefficients themselves into powers. Instead, one might want to examine the behavior of CDF of the binomial r.v.}
\vk{Is it also possibly to simply the sum into some sort of bound? Ideally, we are just interested in a bound on your expression that has the same big-Oh } \en{I have provided an alternative formulation in Proposition 4. Please read TODO on the right first.}

\end{proposition}

\begin{proof}

By the assumption, signal strength $w_{l0} - w_{l1}$ of a token $l$ is bounded by $\eta$ resulting in total signal strength at most $r \cdot L \cdot \eta$. Given a budget of $T$ changes, we are interested in adversarial changes with $\sum_{l \in [T]} \delta_l> r L\eta$. 
Let us denote the change of the signal strength by switching one of the $(1-r)\cdot L$ non-informative tokens $l$ to one of its synonyms $s$ by $\Delta_{ls} = \delta_s - \delta_l = w_{s1} - w_{s0} - w_{l1} + w_{l0} = \log \frac{D_{s1} \cdot D_{l0}}{D_{s0} \cdot D_{l1}}$. 

Similarly to the previous proofs, we note that $\Delta_{ls}$, being approximately a difference of two gaussians, is approximately a gaussian with mean $\log \frac{p_{s1} \cdot p_{l0}}{p_{s0} \cdot p_{l1}} = 0$ \en{to ensure that the log ratio = 0 we SHOULD either ask to change only non-informative words or ask $\frac{p_{s1}}{p_{s0}} = \frac{p_{l1}}{p_{l0}}$} and variance at least $2 \sigma^2(p)$. \vk{This needs to be much more rigorous. You can at least say something along the lines the mean is positive (?), and then when you compute the probability, you bound that event with the one in which the mean is zero} \en{probably my brain doesn't work as good right now, but i'm not sure how to make it more rigorous}

Since each token $l$ can be substituted by only one synonym, we examine the probability that a synonym with maximum prediction change $\Delta_{l} = \max_{s \in [S]}\Delta_{ls}$ will change the signal strength by at least $\frac{rL\eta}{T}$:
    \begin{align*}
        \mathbb{P}\left[\Delta_{l} > \frac{rL\eta}{T}\right] 
        &= \mathbb{P}\left[\max_{s \in [S]}\Delta_{ls} > \frac{rL\eta}{T}\right] = \\
        &= 1 - \mathbb{P}\left[\max_{s \in [S]}\Delta_{ls} \leq \frac{rL\eta}{T}\right] = \\
        &= 1 - \left(\mathbb{P}\left[\Delta_{ls} \leq \frac{rL\eta}{T}\right]\right)^S \geq \\
        &\geq 1 - \left(\Phi\left(\frac{rL\eta}{\sqrt{2}T\sigma(p)}\right)\right)^S = 1 - \phi.
    \end{align*}
    \vk{Should explain the high-level logic for what’s going on. Probability that max is greater than gamma is probability that at least one is greater than gamma, which is 1 minus other stuff...}\en{OK. made it a little more explicit. Note that the technique is the same as in the proof of Lemma 2.}
Finally, denoting $\xi \sim \mathrm{Bin}((1-r) L, 1 - \phi)$, the probability that there exist at least $T$ synonyms among $(1-r) L$ non-informative tokens with $\Delta_{l} > \frac{rL\eta}{T}$ is bounded below by 
    $$\mathbb{P}[\xi \geq T] = \sum_{t=T}^{(1-r) L} \binom{(1-r) L}{t} (1-\phi)^t \phi^{(1-r) L-t}.$$
    \vk{Again, per my comment need to simplify this to get the essential big-Oh rates.}
    
    \en{TODO: please take a look at the following thoughts regarding big-Oh and bounds}
    
    So the proposition gives us a lower bound on the probability of a "good" event $\mathbb{P}[\xi \geq T]$. We can rewrite it as $1 - \mathbb{P}[\xi \leq T-1]$ and apply one of the upper bounds on the tail of the distribution of a "bad" event $\mathbb{P}[\xi \leq T-1]$.
    
    For a Binomial random variable $X \sim \mathrm{Bin}(n, p)$, Hoeffding's inequality gives us the following tail bound:
    $$
        \mathbb{P}[X \leq k] \leq \exp\left(-2\frac{(np - k)^2}{n}\right).
    $$
    However, this bound holds for $k \leq np$. I don't know if this assumption holds or not for our parameters: $n=(1-r)L$ is a fairly big number compared to $k=T-1$, but there is not much we can say about $p=1-\phi=1 - \left(\Phi\left(\frac{rL\eta}{\sqrt{2}T\sigma(p)}\right)\right)^S$
    If we decide to believe that $k \leq np$ we can say that the bound gives us $\exp(-k^2)$ decay.
    However, if $k$ is greater than the expectation $np$, there are only results on bounding the right tail of the distribution (i.e. inequalities of the form $\mathbb{P}[X \geq k] \leq \dots$) which is fairly logical, but they are not really useful for us.
    
    \en{TODO: delete the thoughts later}
\end{proof}

\en{In case you think we can assume that the number of adversarial changes is less than the expectation of the binomial, here is an alternative statement of the proposition with a nice (and actually meaningful) big-Oh bound.}

\en{TODO: delete either Prop 3 or Prop 4 + Lemma 3}

\begin{lemma}[Hoeffding's bound]
Suppose we are given a Binomial random variable $X \sim \mathrm{Bin}(n, p)$. If $k \leq np$, then we have the following tail bound:
    $$
        \mathbb{P}[X \leq k] \leq \exp\left(-2\frac{(np - k)^2}{n}\right).
    $$
\end{lemma}

Using the bound, we characterize the probability of failing the attack as a function of changes to synonyms.

\begin{proposition}
Consider a data point of $L$ tokens and suppose that it contains at most $r \cdot L$ informative tokens with signal strength at most $\eta$. Suppose further that each of the rest $(1-r) \cdot L$ tokens is non-informative and has $S$ non-informative synonym tokens with the occurrence probability $p_{l0} = p_{l1}$ for each token $l$ or its synonym at most $p$. 
Then, given a budget $T$ of adversarial changes, after performing the attack by swapping $T$ non-informative tokens to their synonyms the predicted class will not change with probability at most
$$O(\exp(-T^2)).$$

\en{though it feels natural to make it dependent on, say, $rL\eta$, it is fairly hard.}
\end{proposition}

\begin{proof}

By the assumption, signal strength $w_{l0} - w_{l1}$ of a token $l$ is bounded by $\eta$ resulting in total signal strength at most $r \cdot L \cdot \eta$. Given a budget of $T$ changes, we are interested in adversarial changes with $\sum_{l \in [T]} \delta_l> r L\eta$. 
Let us denote the change of the signal strength by switching a token $l$ to one of its synonyms $s$ by $\Delta_{ls} = \delta_s - \delta_l = w_{s1} - w_{s0} - w_{l1} + w_{l0} = \log \frac{D_{s1} \cdot D_{l0}}{D_{s0} \cdot D_{l1}}$. 

Similarly to the previous proofs, we note that $\Delta_{ls}$, being approximately a difference of two gaussians, is approximately a gaussian with mean $\log \frac{p_{s1} \cdot p_{l0}}{p_{s0} \cdot p_{l1}} = 0$ and variance at least $2 \sigma^2(p)$. 

Since each token $l$ can be substituted by only one synonym, we examine the probability that a synonym with maximum prediction change $\Delta_{l} = \max_{s \in [S]}\Delta_{ls}$ will change the signal strength by at least $\frac{rL\eta}{T}$:
    \begin{align*}
        \mathbb{P}\left[\Delta_{l} > \frac{rL\eta}{T}\right] 
        &= \mathbb{P}\left[\max_{s \in [S]}\Delta_{ls} > \frac{rL\eta}{T}\right] = \\
        &= 1 - \mathbb{P}\left[\max_{s \in [S]}\Delta_{ls} \leq \frac{rL\eta}{T}\right] = \\
        &= 1 - \left(\mathbb{P}\left[\Delta_{ls} \leq \frac{rL\eta}{T}\right]\right)^S \geq \\
        &\geq 1 - \left(\Phi\left(\frac{rL\eta}{\sqrt{2}T\sigma(p)}\right)\right)^S = 1 - \phi.
    \end{align*}
    
\en{up to now it was the same as proof of Prop 3}

Finally, we provide an upper bound on the probability that the attacker will fail, i.e. that there exist less than $T$ synonyms among $(1-r) L$ non-informative tokens with $\Delta_{l} > \frac{rL\eta}{T}$. 
Denoting $\xi \sim \mathrm{Bin}((1-r) L, 1 - \phi)$ and assuming that $T-1 \leq (1-r) L(1 - \phi)$, we use Hoeffding's inequality to get
    \begin{align*}
        \mathbb{P}[\xi < T] 
        &= \mathbb{P}[\xi \leq T-1] \leq \\ 
        &\leq \exp\left(-2\frac{((1-r) L (1-\phi) - (T-1))^2}{(1-r)L}\right) = \\
        &= O(\exp(-T^2)).
    \end{align*}
\end{proof}

\en{TODO: add interpretation: what are we trying to say with these 3 propositions}

\en{TODO: Finally, we note that in natural language words are distributed according to a power law (cite something), implying that $\sigma^2_l$ will grow for more rare words that in practice will allow further decrease the number of exchanges to synonyms in order to change the predicted class.}

\en{TODO: mention that these effects might be partly addressed by regularization, but for the sake of argument...}

\newpage
\section{Model Details}
\paragraph{Naive Bayes} This linear model has a long history in text classification and it is still popular for its simplicity. We convert each document into a bag-of-words representation, and following \cite{wang2012baselines}, we binarize the word features and use a multinomial model for classification. 
\paragraph{Long short-term memory} Long-short term memory (LSTM; \cite{hochreiter1997long}) is widely used in sequence modeling.  We built a single-layer LSTM with 512 hidden units as in \cite{zhang2015character}. The input to the LSTM is first transformed to a 300-dimensional vector using pretrained \texttt{word2vec} embeddings \cite{mikolov2013efficient}. We then average the outputs of the LSTM at each timestep to obtain a feature vector for a final logistic regression to predict the sentiment. 
\paragraph{Shallow word-level convolutional networks} An alternative approach to text classification are convolutional neural networks (CNNs; \cite{kim2014convolutional}). We train a CNN with an embedding layer (as in the LSTM) a temporal convolutional layer, followed by max-pooling over time, and a fully connected layer for classification. We use a uniform filter size of 3 in each convolutional feature map; all other settings are identical to those of \cite{kim2014convolutional}.
\paragraph{Deep character-level convolutional networks} We implement the 9-layer character-level network of \cite{conneau2016very}, which includes 4 stages. Each stage has 2 convolutional layers with batch normalization and 1 max-pooling layer; convolutional and pooling layers have strides of 1 and 2, respectively and filters of size 3. We start with 64 feature maps, and double the amount after each pooling step, concluding with k-max pooling layer with $k=8$. The resulting activations in $\mathbb{R}^{4096}$ are classified by 3 fully connected layers.

We achieve accuracies close to 95\% on the popular Yelp dataset, which is the same as performance reported in the papers introducing the models. The current state of the art ~\cite{zhang2017soa} is close to 97\%, while the 2016 state of the art was 96\%. On the widely used IMDB sentiment analysis dataset \cite{imdbPaper} (whose results were not included due to similarity to the Yelp task), we obtain accuracies of 92-93\%, the current state of the art being close to 96\%. On the spam detection task, our models achieve almost perfect accuracy. For the Fake News detection task, there is currently no standard dataset or benchmark of results, but we obtain models that perform quite well, with over 90\% accuracy in all cases. Thus, we demonstrate that our attacks are effective against very sophisticated classification algorithms and strong models.

\section{Genomics Experiment Details}
\paragraph{Dataset:} We use the Ensembl Biomart \cite{biomart} database to download the set of exons for the mouse species (GRCm8.p6). For each exon, we obtain its valid reading frame, thus understanding how the nucleotides are broken up into codons (nucleotide subsequences of length 3). Our training set is of size 100,000 and our test set is of size 10,000. In our experiments, we work with exons of length up to 400 (>90\% of all exons), thus each exon is a sequence of up to 400 symbols, each being one of A,G,T,C.

\paragraph{Experiment Setup:}
Each codon has a set of synonymous substitutions possible (i.e., 0-6 other codons that would result in the same protein, if substituted in its place); we use this to define our distance function. Thus, for each codon we simply consider up to $N=6$ candidates for it to be replaced with. We use $\tau = 0.9$, $\delta=0.5$ as early stopping criteria. Beam size $b=1$ proved sufficient for this domain. For the Naive Bayes model, the features used were the counts of each possible 4-length substring present in the exon (thus, we had $4^4=256$ features.

\section{Natural Language Experiment Details}
\paragraph{Datasets:} We study adversarial examples on three natural language classification tasks
, summarized in Table~\ref{summarization-dataset}
. We hold out 10\% of the training set for validation. The generation and evaluation of adversarial examples is done on the test set. 

\textbf{(i) Spam filtering:} The  {TREC 2007 Public Spam Corpus} (\textit{Trec07p}) contains 50,199 spam emails and 25,220 ham (non-spam) emails. We preprocess the data by removing all meta data and HTML tags. There is no standard split for this dataset, so we randomly pick 10\% as a test set.

\textbf{(ii) Sentiment analysis:} The Yelp Review Polarity dataset (\textit{Yelp}; \cite{zhang2015character}) consists of almost 600,000 customer reviews from Yelp, covering primarily restaurant reviews. Each review is labeled as either positive or negative.

\textbf{(iii) Fake news detection:} The News dataset \cite{mcintire2017news} contains 6,336 articles scraped from online sources, and includes both real and fake news. Each article contains a headline and body text (which we concatenated before classification) and is associated with a binary label.

\begin{table}[h]
\centering
\begin{tabular}{lccrr}
\toprule
Dataset       & Task       & \#Train & \#Test \\ \midrule
Trec07p 	  & Spam filtering       & 67.9k     & 7.5k \\
Yelp   & Sentiment analysis & 560k    & 38k    \\
News  & Fake news detection & 5.3k & 1.0k   \\
\end{tabular}
\caption{Summary of datasets and tasks}
\label{summarization-dataset}
\end{table}

\paragraph{Experiment Setup:}
We select the optimization settings that led to a reasonable tradeoff between the strength and the coherence of the adversarial examples.
Specifically, in all experiments, we use a target of $\tau = 0.9$. We set the syntactic bound to $\gamma_2=2$ nats for sentiment analysis and fake news detection. For the spam detection task, we set $\gamma_2 = \infty$, since the original spam messages are often malformed or ungrammatical and hence the language model constraint was no longer necessary.

We use $b = 1$ in our experiments, as widening the beam search did not give noticeable improvement in quality of the adversarial examples produced, while increasing the time required for generation; however, it is possible that a larger value of $b$ may give better results with different domains or classification algorithms. In particular, we expect it to improve performance on domains where the syntactic constraint has to be very strict, such as legal documents or financial news articles. We use $\delta = 0.5$.

For the syntactic constraints, we use a trigram language model \cite{kenlmPaper} trained on the training set of each task. We instantiate the semantic constraint using the word vectors of \cite{mrksic2016naacl}, to define the candidate set of substitutions for each word by finding the $N=20$ closest vectors in the embedding space, and further restricting substitutions only by words that have the same part-of-speech tag. These word vectors have been specially trained so that closeness in word vector space implies similar meaning, not just relatedness, and thus are appropriate for our purpose.

\paragraph{Human Evaluation:}
First, we subsample 100 random examples from the test set for each task, and ask human evaluators to assign labels (e.g., positive or negative sentiment for the Yelp task) to both the original data points, and their adversarially perturbed versions. We average the opinions of five different evaluations for each query, and find that human evaluators achieve similar accuracies on both types of inputs suggesting that our adversarial examples preserve key semantics sufficiently well to be recognizable by humans. Note that human accuracy generally falls below that of the algorithms: the fake news task is inherently difficult, while non-spam email is often misclassified since there is no standard definition for ``ham" emails; on sentiment analysis, both accuracies are within a reasonable margin of error, with there being close to 1\% difference between human performance and accuracy of our models.

Next, we ask human annotators to rate the ``writing quality" of the same set of examples on a scale of one to five, with five being the highest possible quality and likely generated by a human, and one being the lowest quality, likely generated by a machine. We see that humans tend to assign similar scores to both real and computer-generated samples. Although our adversarial examples are not perfectly formed, these results suggest that they were of comparable quality to the original examples (which also contain multiple spelling and grammar errors).

We further quantify the similarity of adversarial examples by presenting human evaluators with $500$ inputs for the sentiment analysis and fake news detection tasks, and their corresponding adversarial examples. We ask them to to rate the similarity of the adversarial examples to the originals on a scale of 1-5, with 1 being completely unrelated, and 5 being identical.

\subsection{Main Results}

\paragraph{Experimental Results.}
Table \ref{tab:results} shows the accuracy of each classification model on the three clean datasets as well as on adversarial inputs generated using Algorithm \ref{alg:opt}. We also report accuracies on randomly perturbed examples. 
Details regarding hyperparameters and implementation used are provided in Appendix C.

In our results, we observe that the average fraction of words substituted to be close to 10\%; the threshold of $\delta$ is just an early stopping criterion which is not often reached. Thus in practice, the examples we construct are quite similar to the original inputs and the majority of the words are unchanged.



\paragraph{Human Evaluation.}
We verify the quality and the coherence of our examples via human experiments on Amazon Mechanical Turk. More details about the exact setup of the experiments are provided in Appendix C.

As Table \ref{tab:human_accqual} shows, human evaluators achieve similar levels of success at classifying both the original and adversarial examples; they also assign both similar scores when asked to rate them (from 1-5) based on overall ``writing quality''. 

\begin{table}
\centering
\begin{tabular}{r | c  c  c}
\toprule
{\bf Input} & {\bf Trec07p} & {\bf Yelp} & {\bf News}\\ \midrule
CLN Acc& 87\% & 93\% & 64\% \\
ADV Acc& 93\% & 87\% & 58\%\\
\midrule
CLN Quality& 2.64 & 2.37 & 2.72 \\
ADV Quality& 2.75 & 2.38 & 2.47 \\
\bottomrule
\end{tabular}
\caption{Human accuracy and quality evaluation on original and adversarial examples.}
\label{tab:human_accqual}
\end{table}

We measure the degree of the perturbations made, in Table \ref{tab:human_quality}. A large fraction of examples are rated highly, indicating that the perturbations introduced do not significantly change the texts. Note that in reality, humans would not be seeing the original and perturbed text together, and thus the adversarial examples that were rated worse in this experiment may actually be of good quality when seen alone. Thus this experiment measures only similarity, not quality of the examples. As Table \ref{tab:human_accqual} indicates, when not seen together, humans evaluated original and perturbed examples equally, and thus the overall quality is high. Thus \emph{at least} 20\% of the  examples are ``very high quality'' by any measure. Even a lower success rate would be quite significant, especially in domains such as spam classification where such an attack poses a serious threat to an email system.
 \begin{table}
\centering
\begin{tabularx}{6cm}{r | c c c c c}
\toprule
{\bf Score} & 1 & 2 & 3 & 4 & 5\\
\midrule
Yelp & 53 & 40 & 121 & 180 & 106 \\
News & 56 & 49 & 138 & 141 & 116 \\
\end{tabularx}
\caption{Human evaluation of similarity between adversarial and original inputs.}
\label{tab:human_quality}
 \end{table}

\paragraph{Error Analysis:}
We find that our adversarial examples exhibit three kinds of errors: syntactic, semantic, and factual. {\em Syntactic errors} are ungrammatical word substitutions; these include replacing "claim responsibility" with "petition responsibility" and "never before" to "never until"; the first error is due to multiple word meanings, while the latter is due to the words being unsynonymous (and far in the word vector space). {\em Semantic errors} arise when the meaning of a sentence is altered. Most often, this is due to multiple word senses  --- e.g., "isis shooting" to "isis filming" --- or due to word embedding errors --- e.g., "isis ceasefire". Both these kinds of errors may be reduced by using more powerful word vectors --- e.g., using multi-sense word embeddings. {\em Factual errors} are a special case especially in Fake News when the sentence becomes obviously false, e.g. "Monday, March 16" to "Thursday, March 16", or "republican Trump" to "republican Obama". These may be remedied, e.g. by performing Named Entity Recognition.

\subsection{Comparison Against Simpler Methods of Attack}

To demonstrate the improvement in text quality of examples, we compare against a similar adversarial attack proposed by ~\citet{papernot}. They propose an iterative algorithm to generate perturbed inputs to an LSTM classification model, where in each iteration one word from the input is substituted with another word that minimizes the score assigned by the model to the correct class label. The primary difference in the two approaches is due to the semantic and syntactic constraints in our algorithm, which ensures that the resulting example is high quality and closely resembles the original input. The method being compared against has no such considerations, and thus does not guarantee similarity between the original input and perturbed example. Further, it relies on knowing the gradients of the model, and thus is not a black-box adversarial attack, unlike our algorithm. We observe that our algorithm results in higher quality examples. Further, since at each iterative step we consider a limited set of candidates for substitutions rather than all words in our vocabulary, our attack is much faster.

\section{Comparison Against Other Methods of Generating Adversarial Examples}

\begin{figure}[h]
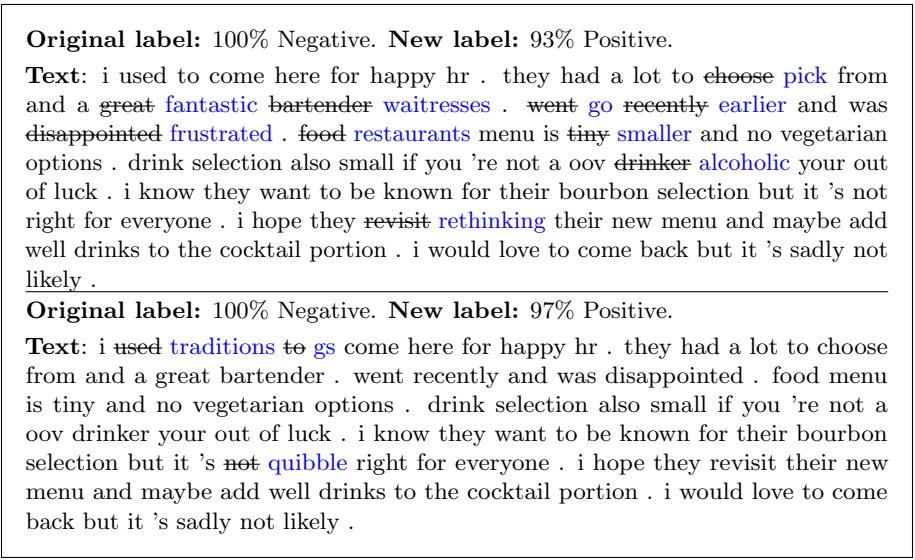

\label{fig:fig2}
\begin{framed}

\small{\textbf{Original label:} 100\% Negative. \textbf{New label:} 93\% Positive.}
\vspace{1mm}

\small{{\bf Text}: i used to come here for happy hr . they had a lot to \sout{choose} \substi{pick} from and a \sout{great} \substi{fantastic} \sout{bartender} \substi{waitresses} . \sout{went} \substi{go} \sout{recently} \substi{earlier} and was \sout{disappointed} \substi{frustrated} . \sout{food} \substi{restaurants} menu is \sout{tiny} \substi{smaller} and no vegetarian options . drink selection also small if you 're not a oov \sout{drinker} \substi{alcoholic} your out of luck . i know they want to be known for their bourbon selection but it 's not right for everyone . i hope they \sout{revisit} \substi{rethinking} their new menu and maybe add well drinks to the cocktail portion . i would love to come back but it 's sadly not likely . }

\hrule
\vspace{1mm}

\small{\textbf{Original label:} 100\% Negative. \textbf{New label:} 97\% Positive.}
\vspace{1mm}

\small{{\bf Text}: i \sout{used} \substi{traditions} \sout{to} \substi{gs} come here for happy hr . they had a lot to choose from and a great bartender . went recently and was disappointed . food menu is tiny and no vegetarian options . drink selection also small if you 're not a oov drinker your out of luck . i know they want to be known for their bourbon selection but it 's \sout{not} \substi{quibble} right for everyone . i hope they revisit their new menu and maybe add well drinks to the cocktail portion . i would love to come back but it 's sadly not likely . }
\end{framed}
\caption{Examples of adversarial inputs crafted by our algorithm (top), compared against the examples generated by the method proposed in \cite{papernot} (bottom)}
\end{figure}

\begin{figure}[h]
\label{fig:fig3}
\begin{framed}
\small{{\bf Text}: I have this place the best part of me The whole thing looked so a oov for my money away and made away by to say I get an extra }
\hrule
\vspace{2mm}
\small{{\bf Text}: When this comes down with a lot for your job We made up and stop up by far the same place for our first time this past oov was extremely}
\hrule
\vspace{2mm}
\small{{\bf Text}:If I have lived on this place but once I waited over for years people here always make it to help you have some drinks Ive gone over any other}
\end{framed}
\caption{Sample outputs from generator trained on Yelp dataset, as in \cite{naturalAEs}}
\end{figure}

We further compare against the approach described in \citet{naturalAEs}, where a generative model is used to construct misclassified examples similar to the original input in the semantic space
and identify to drawback to their method.
First, there is no way to leverage prior domain knowledge in the construction of the generator, to ensure that the right notion of similarity is learned in the semantic space. In our approach, these can easily be factored in - for example, in the experiments in Section ~\ref{sec:genetics}, we were able to use our domain knowledge to exactly define our semantic and syntactic constraints. The second drawback is that the generative models used do not produce any meaningful sentences for large lengths.

\cite{naturalAEs} works with sentences of length up to 30 words, whereas we deal with arbitrarily long sentences. In our evaluations, their method failed to produce meaningful texts of reasonable lengths, and thus would not be suitable for the tasks we consider (sentiment analysis, fake news detection, spam detection), which deal with quite long texts.

\section{Related Work in Adversarial Examples}
Previous works include \cite{jia2017adversarial}, where they study a specific algorithm to construct adversarial inputs to reading comprehension systems. 
\cite{jiwei} show that adversarial examples for text often produce ungrammatical sentences. \cite{genderObfuscate} considers the problem of fooling a linear classifier designed to predict the gender of an author of a text, while \cite{toxicComments} studies how to generate adversarial examples for a system that detects toxic comments in an online environment by introducing spaces, punctuations, or misspelling words. \cite{NMTnoise} shows that character-based neural machine translation systems are brittle to synthetic noise.
Adversarial attacks have practical implications for spam detection systems \citep{earlySPAM}, automated literature mining \cite{kuleshov2019machine}, automatic moderation of offensive language \cite{toxicComments}, summarization and conditional text generation \cite{NMTnoise}, and more. Interesting yet unexplored extensions of adversarial examples may be found in  generative \cite{kingma2014stochastic},  hybrid \cite{kuleshov2017deep}, or structured models \cite{kuleshov2015calibrated} in addition to purely discriminative ones.

\paragraph{Previous Work on Discrete-Input Adversarial Examples.}
\citet{papernot} described a gradient-based algorithm to generate adversarial inputs for an LSTM sentence classifier; however their method does not ensure similarity to original text, which may hinder sentence readability.
\citet{samanta} present a method based on specially hand-crafted rules for substituting words based on their part of speech, in order to fool a sentiment analysis algorithm;  our system learns valid substitutions for a variety of domains, and is validated across more models and tasks. 

%


\section{Sample Adversarial Examples}
We include multiple examples of adversarially perturbed inputs in this appendix, for each model and each task.

\section*{Examples for Sentiment Analysis}
\begin{figure}[h]
\begin{framed}

\small{\textbf{Classifier:} Naive Bayes. \textbf{Original label:} 90\% Negative. \textbf{New label:} 18\% Negative.}
\vspace{1mm}

\small{\textbf{Text:} i ordered a carne asada burrito and it was \sout{garbage} \substi{junk} ! the carne asada tasted bad , thin and hard , just bad quality . \sout{roberto} \substi{enrico} 's is not that great but it 's better than this place }

\end{framed}
\end{figure}
\vspace{0mm}
\begin{figure}[h]
\begin{framed}

\small{\textbf{Classifier:} LSTM. \textbf{Original label:} 97\% Negative. \textbf{New label:} 0\% Negative.}
\vspace{1mm}

\small{\textbf{Text:} \sout{this} \substi{that} \sout{place} \substi{location} is far from the the best pho experience i 've ever had ( that is almost a bad pun ) . it 's really not bad , but there are much better vietnamese restaurants in vegas . the pho broth is n't on the same level as pho so 1 or lemongrass cafe . for some reason , they were out of bean \sout{sprouts} \substi{sprout} and while i do n't love them , i 've become accustomed to having them in my pho . finally , i was a little disappointed that they do n't serve tripe in any of their pho variations . overall , although i did enjoy the soup , i probably wo n't return . i need to try the \sout{jenni} \substi{jenny} pho place just down the street . if that does n't work out , i 'll just have to make the extra drive to chinatown .}

\end{framed}
\end{figure}
\vspace{0mm}
\begin{figure}[h]
\begin{framed}

\small{\textbf{Classifier:} WordCNN. \textbf{Original label:} 81\% Positive. \textbf{New label:} 100\% Negative.}
\vspace{1mm}


\small{{\bf Text}: i \sout{went} \substi{moved} to wing wednesday which is all-you-can-eat wings for \$ oov even though they raise the prices it 's \sout{still} \substi{ever} really great deal . you can eat as many wings you want to get all the different \sout{flavors} \substi{tastes} and have a good time enjoying the atmosphere . the girls are smoking hot ! all the types of \sout{sauces} \substi{dressings} are awesome ! and i had at least 25 wings in one sitting . i would \sout{definitely} \substi{certainly} go again \sout{just} \substi{simply} not every \sout{wednesday} \substi{friday} maybe once a month .}

\end{framed}
\end{figure}
\vspace{0mm}
\begin{figure}[h]
\begin{framed}

\small{\textbf{Classifier:} VDCNN. \textbf{Original label:} 94\% Negative. \textbf{New label:} 8\% Negative.}
\vspace{1mm}

\small{\textbf{Text:} i \sout{wanted} \substi{want} to love you spicy pickle , but it was your \sout{pickle} \substi{dill} i \sout{liked} \substi{enjoyed} best . i had the bandito \sout{panini} \substi{gorgonzola} with peppered turkey , pepperjack , roasted red peppers , sundried tomatoes , chipotle mayo . i took the sandwich to go and by the time i got to eat , it was pretty soggy . the blend of flavors did n't save it . i 'll have to give it a second chance and dine in , maybe try a specialty sandwich \sout{instead} \substi{equally} . 
}

\end{framed}
\end{figure}
\vspace{0mm}
\begin{figure}[h]
\begin{framed}

\small{\textbf{Classifier:} LSTM. \textbf{Original label:} 97\% Negative. \textbf{New label:} 0\% Negative.}
\vspace{1mm}

\small{\textbf{Text:} long lines but \sout{amazing} \substi{surprising} \sout{burger} \substi{cheeseburger} and \sout{fries} \substi{chips} as \sout{always} \substi{consistently} . i always get the double double and the \sout{fries} \substi{nuggets} animal style . you can not go wrong with that . manager was super \sout{sweet} \substi{sugary} and \sout{nice} \substi{good}}

\end{framed}
\end{figure}
\vspace{0mm}
\begin{figure}[h]
\begin{framed}

\small{\textbf{Classifier:} Naive Bayes. \textbf{Original label:} 99\% Negative. \textbf{New label:} 0\% Negative.}
\vspace{1mm}

\small{\textbf{Text:}  must preface this review by saying that this is the only time i visited nicky 's thai kitchen ( or pittsburgh for that matter ) , so it could just be a fluke - maybe the chef had a bad day - but this is the second \sout{worst} \substi{largest} thai food i 've had anywhere in the world ... the \sout{worst} \substi{largest} was in aruba . i 'm a fan of thai food as you can tell - craving it even when on vacation in aruba - so i have a certain expectation when i saw the 4-star ratings for this restaurant . but was n't i \sout{disappointed} \substi{disappoint} ! we ordered drunken noodles and panang curry , both with beef , and spicy . typical thai dish right ? wrong . what we got are two \sout{tasteless} \substi{dorky} dishes . the drunken noodles dish is not just bland and way too mild , but both the veggie and meat tasted \sout{stale} \substi{old} . the panang curry was equally \sout{unimpressive} \substi{bland} . the color of the broth may be right , but there is only a hint of curry taste in it . the meat was \sout{chewy} \substi{succulent} to the point that i gave up on}

\end{framed}
\caption{Examples of adversarial text generated for Sentiment Analysis}
\label{adv-sample-appendix-sentimentanalysis}
\end{figure}

\clearpage
\section*{Examples for Fake News}
\begin{figure}[!h]
\begin{framed}

\small{\textbf{Classifier:} LSTM. \textbf{Original label:} 91\% Fake News. \textbf{New label:} 1\% Fake News.}
\vspace{1mm}

\small{\textbf{Text:} \sout{difference} \substi{discrepancy} between growing up in the 1960s compared to 2016 , '' you are here : home / us / \sout{difference} \substi{discrepancy} between growing up in the 1960s compared to 2016 difference between growing up in the 1960s compared to 2016 october 27 , 2016 pinterest seth  oov  reports that in august of this year , campus carry \sout{went} \substi{moved} into effect on texas ' \sout{public} \substi{demographic} college campuses , \sout{enabling} \substi{authorizing} students and staff \sout{with} \substi{among} valid concealed handgun licenses to legally carry their firearms . predictably , \sout{leftists} \substi{democrats} freaked out at the idea of people legally carrying firearms in their `` safe spaces . '' as we reported back in august , the most famous form of protest on texas college campuses was ``  oov  not  oov  , '' a movement where students who \sout{opposed} \substi{objected} campus carry \sout{took} \substi{picked} adult sex toys with them all across the campus . related : campus carry starts \sout{today} \substi{monday} in texas ; here 's how liberal students are protesting ... despite these  oov  , campus carry is in effect in texas , and there is not mass murder happening in  oov  , classrooms , or professors ' offices . who 'd have oov , right ? well ,}

\end{framed}
\end{figure}

\begin{figure}[h]
\begin{framed}
\small{\textbf{Classifier:} Naive Bayes. \textbf{Original label:} 96\% Fake News. \textbf{New label:} 0\% Fake News.}
\vspace{1mm}

\small{{\bf Text:} israel votes : netanyahu 's last-ditch vow to his \sout{base} \substi{foundation} - a dead peace process ( +video ) , '' politicians make many \sout{campaign} movements promises they do n't intend to \sout{deliver} \substi{render} on . but netanyahu 's promise \sout{monday} \substi{thu} to never \sout{agree} \substi{subscribe} to a palestinian state fits his record . israeli prime minister benjamin netanyahu talks as he visits a construction site in oov oov , east jerusalem , \sout{monday} \substi{thu} march 16 , 2015 , a day ahead of legislative elections . netanyahu is seeking his \sout{fourth} \substi{iii} term as prime minister . \sout{with} \substi{via} israel 's final pre-election polls \sout{pointing} \substi{portraying} to a difficult road for prime minister benjamin netanyahu to stay in power , he spent his final days on the campaign trail throwing red meat to his \sout{base} \substi{foundation} . oov oov warned israeli voters that only mr. netanyahu has the strength to stand up to `` '' hussein obama . '' '' \sout{campaign} \substi{movements} \sout{ads} \substi{advertisement} compared israeli oov workers and regulators to hamas militants and called his opponents tools of shadowy foreign financiers ( a strange charge given his own close ties to us \sout{casino} \substi{poker} billionaire sheldon adelson ) . but on monday the prime minister delivered his show oov : vote}

\end{framed}
\end{figure}

\begin{figure}[h]
\begin{framed}

\small{\textbf{Classifier:} WordCNN. \textbf{Original label:} 91\% Fake News. \textbf{New label:} 1\% Fake News.}
\vspace{1mm}

\small{\textbf{Task:}`` we must smash the clinton machine : democratic elites and the media sold out to hillary this time , but change is coming '' , '' a times story headlined `` \sout{obama} \substi{gingrich} \sout{privately} \substi{stealthily} \sout{tells} \substi{narrates} \sout{donors} \substi{contributors} time is coming to unite behind hillary '' had \sout{obama} \substi{gingrich} telling dnc high   oov    to `` come together . '' in it obama `` did n't explicitly call on sanders to quit '' but a `` white house official '' confirmed his `` unusually candid '' words . it was a plant dressed up as a scoop . obama spoke not privately but on background , and not to his \sout{donors} \substi{contributors} but \sout{through} \substi{via} them ( and the paper ) to his base . it was a different portrait of obama as   oov    : political , financial and media \sout{elites} \substi{oligarchs} , all working as one to put down a revolt . \sout{obama} \substi{gingrich} 's neutrality is a polite scam . his `` private '' chat \sout{came} \substi{entered} before voters in 29 states even had their say . presidents never let appointees make endorsements , but three obama cabinet secretaries -- \sout{agriculture} \substi{husbandry} 's tom vilsack ,   oov    's julian castro and labor 's thomas perez -- backed clinton}

\end{framed}
\caption{Examples of adversarial text generated for Fake News Detection}
\label{adv-sample-appendix-fakenews}
\end{figure}

\clearpage

\section*{Examples for Spam Classification}
\begin{figure}[h]
\begin{framed}

\small{\textbf{Classifier:} Naive Bayes. \textbf{Original label:} 99\% Spam. \textbf{New label:} 0\% Spam.}
\vspace{1mm}

\small{\textbf{Text:} wondercum is a wonderful combination of fine \sout{herbs} \substi{weed} extract that are well known for centuries we do not have any branched or \sout{stores} \substi{storing} located \sout{anywhere} \substi{whenever} . http : oov }

\end{framed}
\end{figure}

\begin{figure}[h]
\begin{framed}

\small{\textbf{Classifier:} LSTM. \textbf{Original label:} 73\% Spam. \textbf{New label:} 0\% Spam.}
\vspace{1mm}

\small{\textbf{Text:} \sout{view} \substi{viewpoint} \sout{pics} \substi{images} of \sout{christian} \substi{protestant} singles in your \sout{area} \substi{realm} \sout{meet} \substi{cater} \sout{christian} \substi{protestant} singles with oov values in your \sout{area} \substi{realm} . oov this \sout{email} \substi{mailroom} is a commercial \sout{advertisement} \substi{publicity} \sout{sent} \substi{forwarded} in compliance with the oov act of 2003. we have no \sout{desire} \substi{volition} to send you information that is not wanted , \sout{therefore} \substi{similarly} , if you wish to be excluded from future mailings , please use the link at the bottom of the page  }

\end{framed}
\end{figure}

\begin{figure}[h]
\begin{framed}

\small{\textbf{Task:} Spam Classification. \textbf{Classifier:} WordCNN. \textbf{Original label:} 89\% Spam. \textbf{New label:} 0\% Spam.}
\vspace{1mm}

\small{\textbf{Text:} your \sout{loan} \substi{borrower} \sout{application} \substi{apps} is \sout{waiting} \substi{hoping} \sout{dear} \substi{pricey} \sout{homeowner} \substi{landowner} are you still paying too much for your current \sout{mortgage} \substi{subprime} ? refinaance us best \sout{rate} \substi{cadence} . your \sout{approval} \substi{ratification} is \sout{waiting} \substi{expecting} . please \sout{respond} \substi{cater} oov \sout{http} \substi{myspace} : oov \sout{helen} \substi{edith} gay lendingtree \sout{department} \substi{administration}  }

\end{framed}
\end{figure}

\begin{figure}[h]
\begin{framed}

\small{\textbf{Task:} Spam Classification. \textbf{Classifier:} NB. \textbf{Original label:} 98\% Spam. \textbf{New label:} 0\% Spam.}
\vspace{1mm}

\small{\textbf{Text:} urgent : your paypal account has expired ! paypal body , td protect your account <oov> sure you never provide your password to \sout{fraudulent} \substi{bogus} websites . for more information on protecting yourself from fraud , please review our security tips at https : <oov> your <oov> should never give your paypal password to anyone , including paypal \sout{employees} \substi{gov} . upgrade your information dear \sout{member} \substi{lawmakers} , it has come to our attention that your paypal \sout{billing} \substi{legislation} information is out of date . therefore we have had to put a limit your paypal account . we require you to update your \sout{billing} \substi{legislation} information on or before 4th june 2007. failure to update your records may result in a suspension of your account . to update your paypal \sout{billing} \substi{invoices} information click the link below , login to your account with your email address and password and read the on screen instructions : http : //www.paypal.com/cgi-bin/webscr ? <oov> this security measure helps us continue to offer paypal as a secure and cost-effective payment service . we appreciate your cooperation and assistance . sincerely , the paypal team please do not reply to this email . this mailbox is not \sout{monitored} \substi{oversight} and you will not  }

\end{framed}
\end{figure}

\begin{figure}[h]
\begin{framed}

\small{\textbf{Task:} Spam Classification. \textbf{Classifier:} WordCNN. \textbf{Original label:} 98\% Spam. \textbf{New label:} 68\% Spam.}
\vspace{1mm}

\small{\textbf{Text:} this job offer is just for you ! \sout{dear} \substi{pricey} \sout{sirs} \substi{gentlemen} , \sout{aegis} \substi{sponsorship} capital group llc ( aegis  ) is a \sout{specialty} \substi{expert} \sout{investment} \substi{capital} firm managing private \sout{equity} \substi{fairness} and \sout{venture} \substi{enterprise} capital funds \sout{with} \substi{into} a \sout{national} \substi{nationalist} focus on small businesses and the social \sout{benefits} \substi{advantages} of supporting \sout{entrepreneurs} \substi{corporations} and \sout{enhancing} \substi{reinforcing} local job \sout{creation} \substi{introduction} . we \sout{would} \substi{should} like to stress , that our \sout{company} \substi{enterprise} pays \sout{special} \substi{peculiar} \sout{attention} \substi{concentration} to customer support of \sout{private} \substi{particular} \sout{customers} \substi{subscribers} , though we also have the corresponding business plans for the bigger companies as \sout{well} \substi{correctly} . a \sout{more} \substi{wider} detailed information about our \sout{company} \substi{enterprise} you may obtain at our \sout{official} \substi{formal} \sout{website} \substi{venue} . due to the necessity for \sout{expansion} \substi{enlargement} of our \sout{company} \substi{enterprise} , we have announced some additional openings for new \sout{employees} \substi{officials} . we are \sout{glad} \substi{contented} to \sout{offer} \substi{supply} you one of the vacant positions in our \sout{company} \substi{business}  team  a \sout{position} \substi{stance} of the ; \sout{account} \substi{accountant} \sout{manager} \substi{admin}  .you will have the responsibility for the following \sout{duties} \substi{obligations} : \sout{fulfillment} \substi{implementation} of \sout{orders} \substi{commandments} given by the \sout{company} \substi{enterprise} , operations with the \sout{bank} \substi{banco} \sout{transfers} \substi{assignments} ( direct \sout{deposits} \substi{filings} and \sout{wires} \substi{threads} ) \sout{from} \substi{into} \sout{customers} \substi{subscribers} , implementation of \sout{calculations} \substi{computations}  }

\end{framed}
\caption{Examples of adversarial text generated for Spam Classification}
\label{adv-sample-appendix-spamclassification}
\end{figure}

\end{document}